\documentclass{llncs}
\usepackage{times}
\usepackage{helvet}
\usepackage{courier}
\usepackage{times}
\usepackage{latexsym}
\usepackage{color}
\usepackage{amssymb}
\usepackage{graphicx}
\usepackage{amsfonts}
\usepackage{amsmath}
\usepackage{mathrsfs}
\usepackage{url}
\usepackage{stmaryrd}
\usepackage{ushort}
\usepackage{multirow}
\usepackage{tikz}
\usepackage{tikz-qtree}
\usepackage{graphicx}
\usepackage{subfigure}

\newcommand{\ba}{\leftrightarrow}

\newenvironment{ContinueExample}{{\noindent {\textbf{Example 1}} ({\it continued.})}\begin{itshape}}{\end{itshape}}

\long\def\comment#1{}

\title{Knowledge Sharing in Coalitions\thanks{This version corrected an error in its previous version published at AI'15. We specially thank Prof.Wojtek Jamroga for pointing out the error.}}%ATL with imperfect information and perfect recall: Communication before acting}%ATL with Perfect Recall and Uniform Strategies based on Distributed Knowledge}%}% Distributed Knowledge, Coalition Ability and Uniform Strategies}
\author{Guifei Jiang\inst{1}\inst{2} \and Dongmo Zhang\inst{1} \and Laurent Perrussel\inst{2}}

\institute{AIRG, Western Sydney University, Penrith, Australia
\and IRIT, University of Toulouse 1, Toulouse, France}
\begin{document}
\maketitle

\begin{abstract}

  The aim of this paper is to investigate the interplay between  knowledge
  shared by a group of agents and its coalition ability. We characterize this
  relation in the standard context of  imperfect information concurrent
  game. We assume that whenever a set of agents
  form a coalition to achieve a goal, they share their knowledge before
  acting. Based on this assumption, we propose new semantics for
  alternating-time temporal logic with imperfect information and perfect
  recall. It turns out that this semantics is sufficient to preserve all
  the desirable properties of coalition ability in traditional coalition
  logics. Meanwhile, we investigate how knowledge sharing within a group of agents
contributes to its coalitional ability through the interplay of epistemic and coalition modalities.
  This work provides a partial answer to the
  question: which kind of group knowledge is required for a group to
  achieve their goals in the context of imperfect information.
\end{abstract}

  \section{Introduction}
Reasoning about coalitional abilities and strategic interactions is fundamental in analysis of multiagent systems (MAS). Among many others~\cite{chatterjee2010strategy,herzig2010dynamic,herzig2006knowing,thielscher2010general,zhang2015logic}, Coalition Logic (CL)~\cite{pauly2002modal} and Alternating-time Temporal Logic (ATL)~\cite{alur2002alternating} are typical logical frameworks that allow to specify and reason about effects of coalitions~\cite{van2006modal}. \comment{In certain sense, CL
can be viewed as a simplification of ATL.}
In a nutshell, these logics express coalition ability using a modality in the form, say $\left \langle\langle G \right \rangle\rangle\varphi$, to mean coalition $G$ (a set of agents) can achieve a property $\varphi$, regardless what the other agents do. ATL/CL assume that each agent in a multi-agent system has complete information about the system at all states (perfect information). Obviously this is not always true in the real world. Different agents might own different knowledge about their system. To model the systems in which agents have imperfect information, a few attempts have been made in the last few years by extending ATL with epistemic operators~\cite{vanDitmarsch14,van2003cooperation,jamroga2004agents,van2005epistemic,guelev2011alternating,huang2016strengthening}. With the extensions, agents' abilities are associated with their knowledge. For instance, assuming a few agents are trying to open a safe, only the ones who know the code have the ability to open the safe. \comment{The strategies used by ATL are thus often unrealistic for systems where global information is unaccessible. To reason about strategies under uncertainty, ATL is combined with epistemic logic and several semantic variants have been proposed based on different interpretations of agents' ability. For instance, an agent may take different strategies or the same strategies at all distinguishable states (uniform strategies); agents may be able to remember the entire history of the game (perfect recall) or just the current state (memoryless, also referred to as imperfect recall)~\cite{jamroga2004agents,schobbens2004alternating}. Also an agents may or may not choose a strategy once and forever (irrevocable strategies)~\cite{aagotnes2007alternating,alur2002alternating}. In this paper, we focus on ATL with imperfect information and perfect recall.}

One difficulty of ATL with imperfect information is how to model knowledge sharing among a coalition. \comment{individual knowledge affects group ability.} In other words, if a group of agents form a coalition, whether their knowledge will be shared and be contributed to the group abilities~\cite{herzig2014logics}?
 \comment{a group $G$ knows how to achieve a
goal $\varphi$, i.e., that $G$ knows the group has a joint strategy to achieve $\varphi$, then it is not clear which
kind of group knowledge is required.}  To the best of our knowledge, most of the existing epistemic ATL-style logics do not assume that members of a coalition share knowledge unless the information is general knowledge~\cite{bulling2014agents,schobbens2004alternating} or common knowledge~\cite{vanDitmarsch14}\comment{ and distributed knowledge~\cite{vanDitmarsch14}} to a group or a system. \comment{The first two types consider a coalition as a set of agents who simply act together to achieve some goal.} However, most of the time when a set of agents form a coalition, their cooperation is not merely limited to acting together, but, more importantly, sharing their knowledge when acting. Safe opening is an example.
\comment{ Polling each member's own knowledge together actually corresponds to assessing distributed knowledge of the coalition.}

This chapter aims to take the challenge of dealing with knowledge sharing among coalitions. By a coalition we mean a set of agents that can not only act together to achieve a goal, but also share their knowledge when acting. We say that a coalition can ensure $\varphi$ if the agents in the coalition distributedly know that they can enforce $\varphi$. Based on this idea, we provide a new semantics for the coalition operator in ATL with imperfect information. It turns out that this semantics is sufficient to preserve
 \comment{
In quest of the ``right'' semantics for perfect recall strategies under uncertainty, Bulling and Jamroga recently proposed a variant of the perfect recall semantics, called truly perfect recall (also referred as no-forgetting semantics)~\cite{bulling2014agents,jamroga2014atl}. This semantics was designed to
%eliminate the counterintuitive effects of the standard perfect recall semantics of ATL: agents may forget the past even though they adapt perfect recall strategies. We show that although this semantics has the advantage to
overcomes the forgetting phenomenon in the standard perfect recall semantics of ATL. Yet it
%is still far from satisfactory. On the one hand, the notion of truly perfect recall is insufficient to deal with situations where different actions may have the same effects, as it requires that an agent just remember the past states without taking actions into consideration; on the other hand, the truly perfect recall semantics fails to preserve
violates
the coalition monotonicity property in traditional coalition logics~\cite{alur2002alternating,goranko2006complete,pauly2002modal}, that is if a coalition can achieve some goal, then its superset can achieve this goal as well. }
desirable properties of coalition ability~\cite{goranko2006complete,pauly2002modal}. More importantly, we investigate how knowledge sharing within a group of agents
contributes to its coalitional ability through the interplay of distributed knowledge and coalition ability. Our contribution is twofold: firstly, this work can be seen as an attempt towards the difficulty: which kind of group knowledge is required for a group to achieve some goal in the context of imperfect information; secondly, these results show that the fixed-point characterizations of coalition operators which normally fail in the context of imperfect information~\cite{Belar14,Belar15} can be partially recovered by the interplay of epistemic and coalitional operators.

The rest of this chapter is structured as follows. Section 2 introduces a motivating example for our new semantics. Section 3 provides the new semantics and investigates its properties. Section 4 explores the interplay of epistemic and coalitional operators. Section 5 discusses related work. Finally we conclude the paper with future work.

\section{A Motivating Example}
Let's consider the following example which highlights our motivation to study coalition abilities under the assumption of knowledge sharing within coalitions.%the counterintuitive properties of existing semantics for ATL with imperfect information~\cite{bulling2014comparing,bulling2014agents,vanDitmarsch14,schobbens2004alternating}.
\begin{example}
 Figure \ref{f} depicts a variant of the shell game \cite{bulling2014agents} with three players: the shuffler s, the guessers $g_1$ and $g_2$. Initially the shuffler places a ball in one of the two shells (the left (L) or the right (R)). The guesser $g_1$ can observe which action the shuffler does, while the other guesser $g_2$ can't. A guesser or a coalition of two wins if she picks up the shell containing the ball. We assume that the guesser $g_1$ takes no action ($n$) and the guesser $g_2$ chooses the shell (the left ($l$) or the right ($r$)).
\end{example}
\begin{figure}
\begin{minipage}[t]{0.5\linewidth}
\centering
 % \centering centering figure
 %\scalebox{0.8} % rescale the figure by a factor of 0.8
 \includegraphics[width=1.5in]{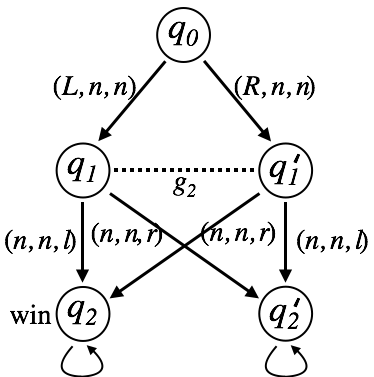} % importing figure
 \caption{{\footnotesize the model $M_1$}}
 \label{f} % labeling to refer it inside the text
 \end{minipage}
\begin{minipage}[t]{0.5\linewidth}
\centering
 %\centering centering figure
%\scalebox{0.8} % rescale the figure by a factor of 0.8
 \includegraphics[width=1.5in]{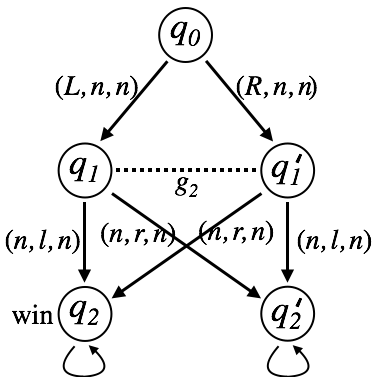} % importing figure
 \caption{{\footnotesize the model $M_2$}}
 \label{f0} % labeling to refer it inside the text
\end{minipage}
{\footnotesize The tuple $(\alpha_1, \alpha_2, \alpha_3)$ represents an action profile,.i.e, action $\alpha_1$ of player s, action $\alpha_2$ of player $g_1$, and action $\alpha_3$ of player $g_2$. The dotted line represents $g_2$'s indistinguishability relation: reflexive loops are omitted. State $q_2$ is labelled with the proposition win.}
\vspace{-5mm}
 \end{figure}
%$n$ indicates the ``do nothing'' action, $L (R)$ denotes the action of placing the ball in the left (right) shell and $l (r)$ denotes the action of picking the left (right) shell.}
Clearly, $g_1$ knows the location of the ball but cannot choose. Instead $g_2$ does not know where the ball is, though he has right to choose the shell. It's easy to see that neither $g_1$ nor $g_2$ can win this game individually. But if $g_1$ and $g_2$ form a coalition, it should follow that by sharing their knowledge they can cooperate to win. However, according to the existing semantics for ATL with imperfect information including the latest one, called truly perfect recall (also referred as no-forgetting semantics)~\cite{bulling2014agents}, the coalition of $g_1$ and $g_2$ does not have such ability to win since they claim that coalition abilities require general knowledge or even common knowledge.

Moreover, these semantic variants fail to preserve the coalition monotonicity which is a desirable property for coalition ability in coalitional logics~\cite{goranko2006complete,pauly2002modal}, that is, if a coalition can achieve some goal, then its superset can achieve this goal as well. For instance, the model $M_2$ in Figure~\ref{f0} depicts a variant game of Example 1 by just switching the available actions of two guessers. The guesser $g_1$ chooses the shell and the guesser $g_2$ takes no action. Then
it's clear that the guesser $g_1$ can win no matter what the others do, as he sees the location of the ball and can pick up the right shell. It should follow that as a group, the guessers $g_1$ and $g_2$ can win this game. However, according to most existing semantics, though the guesser $g_1$ has the ability to win, this ability no longer holds once he forms a coalition with guesser $g_2$. These counterintuitive phenomena motivate our new semantics for ATL with imperfect information and perfect recall.

\section{The Framework}

In this section, we provide a new semantics for ATL with imperfect information and perfect recall based on the assumption of knowledge sharing in coalitions, and then investigate logical properties of ATL under this semantics.
\subsection{Syntax of ATL}
 Let $\Phi$ be a countable set of atomic propositions and $N$ be a finite nonempty set of agents. The language of ATL, denoted by $\mathcal{L}$, is defined by the following grammar:
$$\varphi := p \ |\  \neg \varphi\ |\ \varphi \wedge \varphi\ |\ \left \langle\langle G \right \rangle\rangle\!\bigcirc\!\varphi\ |\ \left \langle\langle G \right \rangle\rangle\Box\varphi\ |\ \left \langle\langle G \right \rangle\rangle\varphi\mathcal{U}\varphi$$
where $p\in\Phi$ and $\emptyset\neq G\subseteq{N}$.

 A coalition operator $\left \langle\langle G \right \rangle\rangle\varphi$ intuitively expresses that
 the group $G$ can cooperate to ensure that $\varphi$. The temporal operator $\Box$ means ``from now on (always)'' and other temporal connectives in ATL are $\mathcal{U}$ (``until'') and $\bigcirc$ (``in the next state''). The dual operator $\Diamond$ of $\Box$ (``either now or at some point in the future'') is defined as $\Diamond\varphi =_{def} \top\mathcal{U}\varphi$. Moreover, the standard epistemic operators can be defined as follows: $K_i\varphi = _{def} \left \langle\langle i\right \rangle\rangle \varphi\mathcal{U}\varphi$ and $D_G\varphi = _{def} \left \langle\langle G\right \rangle\rangle \varphi\mathcal{U}\varphi$. As we will show in the semantics, these abbreviations capture their standard intuitions, i.e., ``$K_i\varphi$'' says the agent $i$ knows $\varphi$, and $D_G\varphi$ means it is distributed knowledge among the group $G$ that $\varphi$. The dual operators of $K$ and  $D$  are defined as follows: $\widehat{K}_i\varphi =_{def} \neg K_i\neg\varphi$, $\widehat{D}_G\varphi =_{def}\neg{D}_G\neg\varphi$.

\subsection{Semantics of ATL}
 The semantics is built upon the \emph{imperfect information concurrent game structure (iCGS)}~\cite{van2003cooperation,schobbens2004alternating}.
 \begin{definition}
 An iCGS is a tuple $M = (N, \Phi, W, \mathcal{A}, \pi, d, \delta, \{R_i\}_{i\in N})$ where
  \begin{itemize}
    \item $N = \{1,2,\cdots,k\}$ is a nonempty finite set of \emph{players};%for some $k >0$
    \item $\Phi$ is a set of \emph{atomic propositions};
    \item $W$ is a nonempty finite set of \emph{states};
    \item $\pi: \Phi\mapsto\wp(W)$ is  a \emph{valuation function};
    \item $\mathcal{A}$ is a nonempty finite set of \emph{actions};
    \item $d: N\times W\mapsto\wp(\mathcal{A})$ is a mapping specifying nonempty sets of actions available to agents at each state. We will write $d_{i}(w)$ rather $d(i,w)$.  The set of joint actions at $w$ for $N$ is denoted as $D(w) = d_{1}(w)\times\cdots\times d_{k}(w)$;
    \item $\delta: W\times D(W)\mapsto W$ is the \emph{transition function} from every pair $(w\in W, \alpha\in D(w))$ to an outcome state $\delta(w, \alpha)\in W$.
    \item $R_i\subseteq W\times W$ is an equivalence relation for agent $i$ indicating the states that are indistinguishable  from her viewpoint. For consistency, we assume that each agent knows which actions are available for her, i.e., $d_i(w) = d_i(w')$ whenever $wR_i w^{\prime}$.
  \end{itemize}
\end{definition}
A {\it path} $\lambda$ is an infinite sequence of states and actions %$w_0w_1w_2\cdots$
$w_0 \stackrel{\alpha_1}{\rightarrow}w_1\stackrel{\alpha_2}{\rightarrow} w_2\cdots$,
where for each $j\geq 1$, $\alpha_j\in{D(w_{j-1})}$ and $\delta(w_{j-1},\alpha_j) = w_{j}$. Any finite segment $w_k \stackrel{\alpha_{k+1}}{\rightarrow}w_{k+1}\stackrel{\alpha_{k+2}}{\rightarrow}\cdots\stackrel{\alpha_{l}}{\rightarrow} w_l$
%$w_kw_{k+1}\cdots w_l$
of a path is called a {\it history}. The set of all histories for $M$ is denoted by $H$. We use $\lambda[j]$ to denote the $j$-th state on path $\lambda$, $\lambda[j,k]$ ($0\leq j\leq k$) to denote the segment of $\lambda$ from the $j$-th state to the $k$-th state, and $\lambda[j, \infty]$ to denote the subpath of $\lambda$ starting from $j$. The length of history $h$, denoted by $|h|$, is defined as the number of actions. %Given $\alpha\in D(w) (= d_{1}(w)\times\cdots\times d_{k}(w))$, let $\alpha_i$ denote the $i$-th component of $\alpha$,

The following definition specifies what a player with perfect reasoning capabilities can in principle know at a special stage of an imperfect information game.

\begin{definition}
Two histories $h = w_{0}\stackrel{\alpha_1}{\rightarrow}w_{1}\stackrel{\alpha_2}{\rightarrow}\cdots \stackrel{\alpha_m}{\rightarrow}w_{m}$ and  $h^\prime = w^\prime_{0}\stackrel{\alpha_1^\prime}{\rightarrow}w^\prime_{1}\stackrel{\alpha_2^\prime}{\rightarrow}\cdots \stackrel{\alpha_n^\prime}{\rightarrow}w^\prime_{n}$ are equivalent for agent $i\in N$, denoted by  $h\approx_i h^\prime$, iff
\begin{enumerate}
  \item $m = n$, %(that is, they have the same length),
  \item $w_jR_iw^\prime_j$ for any $0\leq j\leq m$, and %(that is, the corresponding states are indistinguishable for agent $i$), and
  \item $\alpha_k(i) = \alpha^\prime_k(i)$ for any $1\leq k\leq m$.%(that is, agent $i$ takes the same action at each corresponding stage).
\end{enumerate}
where $\alpha_k(i)$ is the $i$-th component of $\alpha_k$.
\end{definition}
 Intuitively, two histories are indistinguishable for an agent if (1) they have the same length, (2) their corresponding states are indistinguishable for this agent, and (3) the agent takes the same action at the each corresponding stage. Our notion of perfect recall is more like GDL perfect recall \cite{vanDitmarsch14,thielscher2010general} as well as perfect recall in extensive games~\cite{kuhn1953extensive} by requiring that an agent remember the past states as well as her own actions.
This is stronger than the one in most epistemic ATL-style logics which often use the state-based equivalence without taking the actions into consideration, that is, a (truly) perfect recall agent just remembers the past states. Our version has the advantage to deal with situations where different actions may have the same effects. For instance, consider two histories $q_0 \stackrel{a}{\rightarrow}q_1$ and $q_0 \stackrel{b}{\rightarrow}q_1$ with a single agent. According to the state{-}based equivalence, the agent cannot distinguish the two histories, but actually they are different from her view since she takes different actions at state $q_0$\footnote{It is worth to mention that \cite{ruan2012strategic} proposed a way to embed actions to a state so that the state-based equivalence can achieve the same meaning.}.
Note that the perfect recall agent does not observe or remember other agents' actions.

 In particular, we say two paths $\lambda$ and  $\lambda^\prime$ are equivalent up to stage $j\geq 0$ for agent $i\in N$, denoted by $\lambda\approx_i^j\lambda^\prime$, iff $\lambda[0,j]\approx_i\lambda^\prime[0,j]$. As mentioned before, we assume that whenever a set of agents form a coalition to achieve their goals, the agents share their own knowledge before acting. To make this idea precise, we extend the indistinguishability relation $\approx_i^j$ to groups as the intersection of all its members' individual equivalence relation, i.e., $\approx_G ^j=$ $\bigcap_{i\in G}\approx^j_i$. Let $\approx_G^j(\lambda)$ denote the set of all paths that are indistinguishable from $\lambda$ up to stage $j$ for coalition $G$, i.e., $\approx_G^j(\lambda) = \{\lambda'~|~ \lambda\approx_G^j\lambda'\}$.

A {\it strategy} is a plan telling one agent what to do at each stage of a given game. %We say a strategy of agent $i\in N$ is {\it uniform} if the strategy specifies the same action for agent $i$ at all indistinguishable histories.
With knowledge sharing among members of a coalition $G\subseteq N$, we say a strategy of agent $i\in G$ is {\it uniform } if the strategy specifies the same action for $i$ at all histories which are indistinguishable for $G$. Formally,
\begin{definition}
Given $i\in G\subseteq N$, a uniform perfect recall strategy for agent $i$ w.r.t $G$ is a function $f_i: H\rightarrow\mathcal{A}$ such that for any history $h, h'\in H$,
\begin{enumerate}
\item  $f_i(h)\in d_i(last(h))$, and %(that is, it specifies a legal action for each history), and
\item if $h\approx_Gh'$, then $f_i(h)= f_i(h')$,% (that is, it specifies the same action for indistinguishable histories).
\end{enumerate}
where $last(h)$ denotes the last state of $h$.
\end{definition}
Intuitively, a uniform perfect recall strategy for an agent in a group tells one of her legal actions to take at each history and specifies the same action for indistinguishable histories of the group. In particular, the standard notion of uniform strategies with respect to individual knowledge can be viewed as a special case.
In the rest of paper, we simply call a uniform perfect recall strategy a strategy.

A {\it joint strategy} for group $\emptyset\neq G\subseteq N$, denoted by $F_G$, is a vector of its members' individual strategies, i.e., $\langle f_i\rangle_{i\in G}$.
Function $\mathcal{P}(h, f_i)$ returns the set of all paths that can occur when agent $i$'s strategy $f_i$ executes after an initial history $h$. Formally, $\lambda\in \mathcal{P}(h, f_i)$ iff $\lambda[0, |h|] = h$ and for any $j\geq |h|$, $ f_i(\lambda[0,j]) = \theta_i(\lambda, j)$ where $\theta_i(\lambda, j)$ is the action of agent $i$ taken at stage $j$ on path $\lambda$. Obviously,  the set of all paths complying with joint strategy $F_G$ after $h$ is defined as $\mathcal{P}(h, F_G) = \bigcap_{i\in G}\mathcal{P}(h, f_i)$.  It should be noted that if a group is characterized by full coordination both on the level of strategies and knowledge, we may view the group as a single agent whose abilities and knowledge are the sum of those of all the members~\cite{kazmierczak2014multi}.

We are now in the position to introduce the new semantics for ATL. Formulae are interpreted over triples consisting of a model, a path and an index which indicates the current stage on the path.
\begin{definition}
Let $M$ be an iCGS. Given a path $\lambda$ of $M$ and a stage $j\in \mathbb{N}$ on $\lambda$,
the \emph{satisfiability} of a formula~$\varphi$ wrt.~$M$, $\lambda$ and $j$, denoted by $M,\lambda, j\models\varphi$, is defined as follows\emph{:}
\vspace{-1mm}
\begin{tabbing}
---\=------------------------------------\=--------\=\kill
\>$M,\lambda, j\models p$                       \>iff\> $p\in \pi(\lambda[j])$\\
\>$M,\lambda, j\models \neg \varphi$               \>iff\> $M,\lambda, j\not\models \varphi$\\
\>$M,\lambda, j\models  \varphi_1 \wedge \varphi_2$      \>iff\> $M,\lambda, j\models \varphi_1$ and $M,\lambda, j\models  \varphi_2$\\
\>$M,\lambda, j\models \left\langle\langle G \right \rangle\rangle\!\bigcirc\!\varphi$                       \>iff\> $\exists F_G$ $\forall\lambda^\prime\in \approx_G^j(\lambda)$ $\forall\lambda''\in \mathcal{P}(\lambda'[0,j], F_G)$\\
\> \> \>$M, \lambda'', j+1\models\varphi$\\
\>$M,\lambda, j\models \left\langle\langle G \right \rangle\rangle\!\Box\!\varphi$                      \>iff\> $\exists F_G$ $\forall\lambda^\prime\in \approx_G^j(\lambda)$ $\forall\lambda''\in \mathcal{P}(\lambda'[0,j], F_G)$ \\
\> \> \>$\forall k\geq j$ $M, \lambda'', k\models\varphi$\\
\>$M,\lambda, j\models \left \langle\langle G \right \rangle\rangle\varphi_{1}\mathcal{U}\varphi_{2}$                      \>iff\> $\exists F_G$ $\forall\lambda^\prime\in\approx_G^j(\lambda)$$\forall\lambda''\in \mathcal{P}(\lambda'[0,j], F_G)$\\
\> \> \> $\exists k\geq j$, $M, \lambda'', k\models\varphi_2$, and\\
\> \> \> $\forall j\leq t < k$, $M, \lambda'', t\models\varphi_1$\\
\end{tabbing}
\end{definition}
\vspace{-3mm}
The interpretation for the coalition operator $\left \langle\langle G \right \rangle\rangle\varphi$ captures its precise meaning that the coalition $G$ by sharing knowledge can cooperate to enforce that $\varphi$. Alternatively, the agents in $G$ distributedly know that they can enforce that $\varphi$. A formula $\varphi$ is \emph{valid} in an iCGS $M$, written as $M \models\varphi$, if $M, \lambda, j\models\varphi$ for all paths $\lambda\in M$ and every stage $j$ on $\lambda$. A formula $\varphi$ is {\it valid}, denoted by $\models\varphi$, if it is valid in every iCGS $M$.
 %Intuitively, in ATL with imperfect information, the coalition operator $\left \langle\langle G \right \rangle\rangle\varphi$ is read as ``Agents in coalition $G$ distributedly know that they can enforce $\varphi$''.

 %Similar to \cite{jamroga2014atl}, the standard epistemic operators can be expressed in ATL with imperfect information. Specifically, $K_i\varphi = _{def} \left \langle\langle i\right \rangle\rangle \varphi\mathcal{U}\varphi$ and $D_G\varphi = _{def} \left \langle\langle G\right \rangle\rangle \varphi\mathcal{U}\varphi$.
 We first show that, as we claimed before, the abbreviations capture the intended meanings of the epistemic operators.
\begin{proposition}
Given an iCGS $M$, a path $\lambda$ of $M$ and a stage $j\in \mathbb{N}$ on $\lambda$,
\begin{itemize}
  \item $M,\lambda, j\models K_i\varphi$ ~ iff ~ for all $\lambda'\approx_i^j\lambda$, $M, \lambda', j\models\varphi$.
  \item $M,\lambda, j\models D_G\varphi$ ~ iff ~ for all $\lambda'\in \approx_G^j(\lambda)$, $M, \lambda', j\models\varphi$.
\end{itemize}
\end{proposition}
\begin{proof}
  We just prove the first clause, and the second one is proved in a similar way.
  It suffices to show that $M,\lambda, j\models \left\langle\langle i \right \rangle\rangle\varphi\mathcal{U}\varphi$ iff for all $\lambda'\approx_i^j\lambda$, $M, \lambda', j\models\varphi$. The direction from the right to the left is straightforward according to the truth condition for $\mathcal{U}$. We next show the other direction. Suppose $M,\lambda, j\models \left\langle\langle i \right \rangle\rangle\varphi\mathcal{U}\varphi$ and for all $\lambda'\approx_i^j\lambda$, then there is $f_i$ such that for any $\lambda''\in \mathcal{P}(\lambda'[0,j], f_i)$, $M, \lambda'', j\models\varphi$.
  And $\lambda'[0,j] = \lambda''[0,j]$, so  $M, \lambda', j\models\varphi$.
\end{proof}

We demonstrate with the variant shell game that the new semantics justifies our intuitions that the coalition of two guessers by sharing their knowledge can win the game.% as the guesser $g_1$ knows the location of the ball.

\medskip

\begin{ContinueExample}
Consider the model $M_1$ in Figure~\ref{f}.
It's easy to check that at the stage 1 on the left path $\lambda_1:= q_0q_1q_2\cdots$, neither guesser $g_1$ nor guesser $g_2$ has the ability to win at the next stage, i.e., $M_1, \lambda_1, 1\not\models\left \langle\langle g_1\right \rangle\rangle\!\bigcirc\! win$ and $M_1, \lambda_1, 1\not\models\left \langle\langle g_2\right \rangle\rangle\!\bigcirc\! win$. Instead when $g_1$ and $g_2$ form a coalition, after sharing knowledge, the guesser $g_2$ is able to distinguish the history $q_0q_1$ from the history $q_0q'_1$, then they can cooperate to win, i.e., $M_1, \lambda_1, 1\models\left \langle\langle\{g_1,g_2\}\right \rangle\rangle\!\bigcirc\! win$.

For the coalition monotonicity property, consider the model $M_2$ in Figure~\ref{f0}. It's easy to check that at the stage 1 on the left path $\lambda_1:= q_0q_1q_2\cdots$, guesser $g_1$ has the ability to win at the next stage by choosing the left shell, i.e.,
    $M_2, \lambda_1, 1\models\left \langle\langle g_1\right \rangle\rangle\!\bigcirc\! win$. Moreover, when $g_1$ and $g_2$ form a coalition, then they can cooperate to win, i.e.,
    $M_2, \lambda_1, 1\models\left \langle\langle\{g_1,g_2\}\right \rangle\rangle\!\bigcirc\! win$.
%Thus, $M_1, \lambda_1, 1\models\left \langle\langle g_1\right \rangle\rangle\!\bigcirc\! win \rightarrow \left \langle\langle\{g_1,g_2\}\right \rangle\rangle\!\bigcirc\! win$.
 \end{ContinueExample}

\medskip

It should be noted that the reason why alternative semantics~\cite{bulling2014agents,vanDitmarsch14,schobbens2004alternating} fail to keep the coalition monotonicity property is that their interpretations of coalition operators $\left\langle\langle G \right \rangle\rangle\varphi$ are given with respect to either the union of each member's equivalence relation or its transitive reflexive closure. This means that the coalition ability implicitly requires  general knowledge or common knowledge of the group, while neither of them is coalitionally monotonic. Instead distributed knowledge is sufficient for coalition ability to preserve the coalition monotonicity property.

\subsection{Properties of the New Semantics}
 We first show that the new semantics satisfied the desirable properties of coalition ability in traditional coalitional logics~\cite{goranko2006complete,pauly2002modal}.
 \begin{proposition}
 For any $G$, $G_1$, $G_2\subseteq N$ and any $\varphi, \psi\in \mathcal{L}$,
\begin{enumerate}
   \item $\models \neg\left\langle\langle G \right \rangle\rangle\!\bigcirc\bot$ \label{false}
   \item $\models \left\langle\langle G \right \rangle\rangle\!\bigcirc\top$ \label{true}
   \item $\models \left\langle\langle G \right \rangle\rangle\!\bigcirc\!(\varphi\wedge\psi)\rightarrow\left\langle\langle G \right \rangle\rangle\!\bigcirc\!\varphi$ \label{fm}
   \item $\models \left\langle\langle G_1 \right \rangle\rangle\!\bigcirc\!\varphi\rightarrow\left\langle\langle G_2 \right \rangle\rangle\!\bigcirc\!\varphi$ where $G_1\subseteq G_2$ \label{cm}% derived by above formula with $[G]\!\bigcirc\!\top$
   \item $\models \left\langle\langle G_1 \right \rangle\rangle\!\bigcirc\!\varphi\wedge\left\langle\langle G_2 \right \rangle\rangle\!\bigcirc\!\psi\rightarrow\left\langle\langle G_1\cup G_2 \right \rangle\rangle\!\bigcirc\!(\varphi\wedge\psi)$ where $G_1\cap G_2 = \emptyset$ \label{super}
  \item $\models \left\langle\langle G \right \rangle\rangle\!\bigcirc\!\varphi\rightarrow\neg\left\langle\langle N\backslash G \right \rangle\rangle\!\bigcirc\!\neg\varphi$\label{regular}
  \end{enumerate}
  Similarly for the $\Box$ and $\mathcal{U}$ operators.
\end{proposition}
Clause \ref{false} says that no coalition $G$ can enforce the falsity while \ref{true} states every coalition $G$ can enforce the truth.
\ref{fm} and \ref{cm} specify the outcome-monotonicity and the coalition-monotonicity, respectively. \ref{super} is the superadditivity property specifying disjoint coalitions can combine their strategies to achieve more. \ref{regular} is called $G$-regularity specifying that it is impossible for a coalition and its complementary set to enforce inconsistency.

The next proposition provides interesting validities about epistemic and coalitional operators.
\begin{proposition} \label{p3}
 For any $G\subseteq N$ and any $\varphi, \psi\in \mathcal{L}$,
 \begin{enumerate}
   \item $\models\left\langle\langle G \right \rangle\rangle\!\bigcirc\!\varphi\ba\left\langle\langle G \right \rangle\rangle\!\bigcirc\! D_G\varphi$ \label{cd1}
   \item  $\models\left\langle\langle G \right \rangle\rangle\!\bigcirc\!\varphi\ba D_G\left\langle\langle G \right \rangle\rangle\!\bigcirc\!\varphi$ \label{cd2}
   \item $\models\left\langle\langle G \right \rangle\rangle\Box\varphi\ba\left\langle\langle G \right \rangle\rangle\Box D_G\varphi$ \label{cd3}
   \item  $\models\left\langle\langle G \right \rangle\rangle\Box\varphi\ba D_G\left\langle\langle G \right \rangle\rangle\Box\varphi$ \label{cd4}
   \item $\models\left\langle\langle G \right \rangle\rangle D_G\varphi\mathcal{U}D_G\psi\rightarrow \left\langle\langle G \right \rangle\rangle\varphi\mathcal{U}\psi$
   \item $\models\left\langle\langle G \right \rangle\rangle\varphi\mathcal{U}\psi\ba D_G\left\langle\langle G \right \rangle\rangle\varphi\mathcal{U}\psi$ \label{cd7}
 \end{enumerate}
\end{proposition}

\begin{proof}
We only give proof for the first two clauses and the proof for $\Box$, $\mathcal{U}$ is similar.

\ref{cd1}. For every iCGS $M$, every path $\lambda$ of $M$ and every stage $j\in \mathbb{N}$ on $\lambda$, assume $M, \lambda, j\models\left\langle\langle G \right \rangle\rangle\!\bigcirc\!\varphi$, then there is $F_G = \langle f_i\rangle_{i\in G}$ such that for all $\lambda'\in\approx_G^j(\lambda)$, for all $\lambda''\in\mathcal{P}(F_G, \lambda'[0,j])$, $M, \lambda'', j+1\models\varphi$. We next show that $F_G$ is the joint strategy to verify $\left\langle\langle G \right \rangle\rangle\!\bigcirc\! D_G\varphi$. Suppose for a contradiction that there is $\lambda_1\in \approx_G^j(\lambda)$, there is $\lambda_2\in\mathcal{P}(F_G, \lambda_1[0,j])$, there is $\lambda_3\in\approx_G^{j+1}(\lambda_2)$ such that $M, \lambda_3, j+1\not\models\varphi$. Then $\lambda_3\in\approx_G^j(\lambda)$ and $\theta_i(\lambda_3,j) = \theta_i(\lambda_2,j) = f_i(\lambda_2[0,j])$ for every $i\in G$, so there is some $\lambda^*\in \bigcup_{\lambda'\in\approx_G^j(\lambda)}\mathcal{P}(F_G, \lambda'[0,j])$ such that $\lambda^*[0,j+1] = \lambda_3[0,j+1]$. And by assumption we have $M, \lambda^*, j+1\models\varphi$. It follows that $M, \lambda_3, j+1\models\varphi$: contradiction. Thus,  $M, \lambda, j\models\left\langle\langle G \right \rangle\rangle\!\bigcirc\! D_G\varphi$. The other direction is straightforward.

\ref{cd2}. For every iCGS $M$, every path $\lambda$ of $M$ and every stage $j\in \mathbb{N}$ on $\lambda$, assume $M, \lambda, j\models\left\langle\langle G \right \rangle\rangle\!\bigcirc\!\varphi$, then there is $F_G = \langle f_i\rangle_{i\in G}$ such that for all $\lambda'\in\approx_G^j(\lambda)$, for all $\lambda''\in\mathcal{P}(F_G, \lambda'[0,j])$, $M, \lambda'', j+1\models\varphi$. We next prove that for any $\lambda^*\in\approx_G^j(\lambda)$, $M, \lambda^*, j\models\left\langle\langle G \right \rangle\rangle\!\bigcirc\!\varphi$. We consider the strategy $F_G$ and it is easy to check that for all $\lambda_1\in\approx_G^j(\lambda^*)$, for all $\lambda_2\in\mathcal{P}(F_G, \lambda_1[0,j])$, $M, \lambda_2, j+1\models\varphi$ as $\approx_G^j(\lambda^*) = \approx_G^j(\lambda)$. Thus, $M, \lambda, j\models D_G\left\langle\langle G \right \rangle\rangle\!\bigcirc\!\varphi$. The other direction is straightforward.
\end{proof}

Note that it is not generally the case that $\models\left\langle\langle G \right \rangle\rangle\varphi\mathcal{U}\psi\rightarrow \left\langle\langle G \right \rangle\rangle D_G\varphi\mathcal{U}D_G\psi$.
Here is a counter-example. Consider the model $M_3$ in Figure \ref{f2} with two agents 1 and 2 and states $\{q_0, q_1, q'_1, q_2, q'_2\}$, where $q_1R_1q'_1$, but not for 2, and all the other states can be distinguished by both agents. There are two propositions $p$, $q$, and $\pi(p) = \{q_1\}$, $\pi(q) = \{q'_1, q_2\}$. The transitions are depicted in Figure \ref{f2}. Consider the left path $\lambda_1:= q_0q_1q_2\cdots$. It is easy to check that $M_3, \lambda_1, 1\models \left\langle\langle 1 \right \rangle\rangle p\mathcal{U}q$, but $M_3, \lambda_1, 1\not\models \left\langle\langle 1 \right \rangle\rangle K_1p\mathcal{U}K_1q$.
\begin{figure}
\begin{minipage}[t]{0.5\linewidth}
\centering
 % \centering centering figure
 %\scalebox{0.8} % rescale the figure by a factor of 0.8
 \includegraphics[width=1.33in]{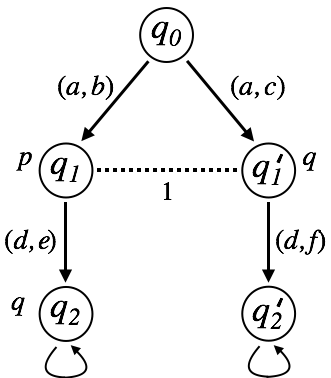} % importing figure
 \caption{{\footnotesize the counter-model $M_3$}}
 \label{f2} % labeling to refer it inside the text
 \end{minipage}
\begin{minipage}[t]{0.5\linewidth}
\centering
 %\centering centering figure
%\scalebox{0.8} % rescale the figure by a factor of 0.8
 \includegraphics[width=1.33in]{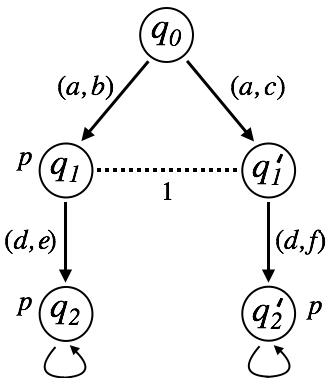} % importing figure
 \caption{{\footnotesize the counter-model $M_4$}}
 \label{f1} % labeling to refer it inside the text
\end{minipage}
\vspace{-4mm}
 \end{figure}

It follows from Proposition \ref{p3} that the distributed knowledge operator and the coalition operator are interchangeable w.r.t temporal operators $\bigcirc$ and $\Box$.
\begin{corollary}
For any $G\subseteq N$ and any $\varphi\in \mathcal{L}$,
  \begin{itemize}
    \item $\models\left\langle\langle G \right \rangle\rangle\!\bigcirc\! D_G\varphi\ba D_G\left\langle\langle G \right \rangle\rangle\!\bigcirc\!\varphi$ \label{cd5}
   \item $\models\left\langle\langle G \right \rangle\rangle\Box D_G\varphi\ba D_G\left\langle\langle G \right \rangle\rangle\Box\varphi$ \label{cd6}
  \end{itemize}
\end{corollary}
\comment{

We finally show that the change in semantics has important consequences for the resulting logic. We compare the validity sets induced by the new semantics and the most similar one, called no-forgetting semantics~\cite{bulling2014agents}. The no-forgetting semantics was recently designed to eliminate the counterintuitive effects of the standard perfect recall semantics of ATL: agents may forget the past even though they adapt perfect recall strategies.
Intuitively, each formula can be interpreted as a property of a game, and then a valid formula specifies a general property of a game that universally holds. Thus, by comparing validity sets of different semantics, we are able to compare the general properties of games induced by the semantics~\cite{jamroga2011comparing}. Specifically, we have the following proposition.
\begin{proposition}
For any formula $\varphi\in\mathcal{L}$, if $\models^{nf}\varphi$, then $\models\varphi$.
\end{proposition}

\begin{proof}
  This is proved by induction on the structure of $\varphi$.
\end{proof}

In the light of the coalition monotonicity property, it is not the case for the converse. Thus, under imperfect information the new semantics describes a more specific class of games than the no-forgetting semantics.
}
\section{The Fixed-point Characterization}
In this section, we will investigate the interplay between knowledge shared by a group of agents and its coalition ability in ATL with imperfect information and perfect recall. We first show that, similar to~\cite{Belar14,Belar15}, the standard fixed{-}point characterizations of coalition operators for ATL~\cite{goranko2006complete} fail under our new semantics.
\begin{proposition}
For any $G\subseteq N$ and any $\varphi, \psi\in \mathcal{L}$,
\begin{itemize}
  \item $\not\models\varphi\wedge\left\langle\langle G \right \rangle\rangle\!\bigcirc\!\left\langle\langle G \right \rangle\rangle\Box\varphi\rightarrow \left\langle\langle G \right \rangle\rangle\Box\varphi$
  \item $\not\models\varphi\vee\left\langle\langle G \right \rangle\rangle\!\bigcirc\!\left\langle\langle G \right \rangle\rangle\Diamond\varphi\rightarrow\left\langle\langle G \right \rangle\rangle\Diamond\varphi$
  \item $\not\models\left\langle\langle G \right \rangle\rangle\Diamond\varphi\rightarrow\varphi\vee\left\langle\langle G \right \rangle\rangle\!\bigcirc\!\left\langle\langle G \right \rangle\rangle\Diamond\varphi$
  \item $\not\models\psi\vee(\varphi\wedge\left\langle\langle G \right \rangle\rangle\!\bigcirc\!\left\langle\langle G \right \rangle\rangle\varphi\mathcal{U}\psi)\rightarrow\left\langle\langle G \right \rangle\rangle\varphi\mathcal{U}\psi$
  \item $\not\models\left\langle\langle G \right \rangle\rangle\varphi\mathcal{U}\psi\rightarrow\psi\vee(\varphi\wedge\left\langle\langle G \right \rangle\rangle\!\bigcirc\!\left\langle\langle G \right \rangle\rangle\varphi\mathcal{U}\psi)$
   %\item $[\emptyset]\Box(\psi\rightarrow(\varphi\wedge[G]\!\bigcirc\!\psi))\rightarrow[\emptyset]\Box(\psi\rightarrow[G]\Box\varphi)$
\end{itemize}
\end{proposition}
Here is a counter-example for the first one. Consider the model $M_4$ in Figure \ref{f1} which is obtained from $M_3$ by just changing the valuations.  %with two agents 1 and 2 and states $\{q_0, q_1, q'_1, q_2, q'_2\}$, where $q_1R_1q'_1$, but not for 2, and all the other states can be distinguished by both agents.
There is one proposition $p$, and $\pi(p) = \{q_1, q_2, q'_2\}$.
Consider $\varphi:= p$ and the left path $\lambda_1:= q_0q_1q_2\cdots$. Then it is easy to check that $M_4, \lambda_1, 1\models p$ and
$M_4, \lambda_1, 1\models \left\langle\langle 1 \right \rangle\rangle\!\bigcirc\!\left\langle\langle 1 \right \rangle\rangle\Box p$, but
$M_4, \lambda_1, 1\not\models \left\langle\langle 1 \right \rangle\rangle\Box p$. Thus, $M_4, \lambda_1, 1\not\models p\wedge\left\langle\langle 1 \right \rangle\rangle\!\bigcirc\!\left\langle\langle 1 \right \rangle\rangle\Box p\rightarrow \left\langle\langle 1 \right \rangle\rangle\Box p$.

On the other hand, we have the following proposition showing that the converse direction for $\Box$ holds under the new semantics.
\begin{proposition}\label{rfp}
For any $G\subseteq N$ and any $\varphi\in \mathcal{L}$,
$\models\left\langle\langle G \right \rangle\rangle\Box\varphi\rightarrow\varphi\wedge\left\langle\langle G \right \rangle\rangle\!\bigcirc\!\left\langle\langle G \right \rangle\rangle\Box\varphi$
  %\begin{itemize}
  %\item $\models\left\langle\langle G \right \rangle\rangle\Box\varphi\rightarrow\varphi\wedge\left\langle\langle G \right \rangle\rangle\!\bigcirc\!\left\langle\langle G \right \rangle\rangle\Box\varphi$
  %\item $\models\left\langle\langle G \right \rangle\rangle\Diamond\varphi\rightarrow\varphi\vee\left\langle\langle G \right \rangle\rangle\!\bigcirc\!\left\langle\langle G \right \rangle\rangle\Diamond\varphi$
  %\item $\models\left\langle\langle G \right \rangle\rangle\varphi\mathcal{U}\psi\rightarrow\psi\vee(\varphi\wedge\left\langle\langle G \right \rangle\rangle\!\bigcirc\!\left\langle\langle G \right \rangle\rangle\varphi\mathcal{U}\psi)$
   %\item $[\emptyset]\Box(\psi\rightarrow(\varphi\wedge[G]\!\bigcirc\!\psi))\rightarrow[\emptyset]\Box(\psi\rightarrow[G]\Box\varphi)$
%\end{itemize}
\end{proposition}
%\comment{
\begin{proof}
  For every iCGS $M$, every path $\lambda$ of $M$ and every stage $j\in \mathbb{N}$ on $\lambda$, assume $M, \lambda, j\models \left\langle\langle G \right \rangle\rangle\Box\varphi$, then there is $F_G = \langle f_i\rangle_{i\in G}$ such that for all $\lambda'\in \approx_G^j(\lambda)$, for all $\lambda''\in \mathcal{P}(F_G, \lambda'[0,j])$, for all $k\geq j$ $M, \lambda'', k\models\varphi$.
  In particular, $\lambda\in \approx_G^j(\lambda)$, then for all $\lambda^*\in \mathcal{P}(F_G, \lambda[0,j])$,  $M, \lambda^*, j\models\varphi$. And by $\lambda[0,j] = \lambda^*[0, j]$, so $M, \lambda, j\models \varphi$.

We next prove that $M, \lambda,j\models\left\langle\langle G \right \rangle\rangle\!\bigcirc\!\left\langle\langle G \right \rangle\rangle\Box\varphi$. It suffices to show that $F_G$ is just the joint strategy for the both coalition operators. That is, for all $\lambda_1\in\approx_G^j(\lambda)$, for all $\lambda_2\in\mathcal{P}(F_G,\lambda_1[0,j])$, for all $\lambda_3\in\approx_G^{j+1}(\lambda_2)$, for all $\lambda_4\in\mathcal{P}(F_G,\lambda_3[0,j+1])$, we want to prove that for all $r\geq j+1$ $M, \lambda_4, r\models \varphi$. As $\lambda_4\in\mathcal{P}(F_G,\lambda_3[0,j+1])$, then $\lambda_4[0, j+1] = \lambda_3[0, j+1]$, then $\lambda_4\in\approx_G^{j+1}(\lambda_2)$, then $\lambda_4\in\approx_G^{j}(\lambda_2)$ and $\lambda_4\in\bigcup_{\lambda_1\in\approx_G^j(\lambda)}\mathcal{P}(F_G,\lambda_1[0,j])$. So by the assumption we have that $M, \lambda_4, t\models\varphi$, so $M, \lambda,j\models\left\langle\langle G \right \rangle\rangle\!\bigcirc\!\left\langle\langle G \right \rangle\rangle\Box\varphi$.

 Thus, $M, \lambda,j\models \varphi\wedge\left\langle\langle G \right \rangle\rangle\!\bigcirc\!\left\langle\langle G \right \rangle\rangle\Box\varphi$.
\end{proof}
%}

 We now present the main result about the interactions of group knowledge and coalition ability for ATL with imperfect information and perfect recall. Recall that $\widehat{K}$ and $\widehat{D}$ be the dual operators of $K$ and $D$, respectively. %That is, $\widehat{K}_i\varphi =_{def} \neg K_i\neg\varphi$, $\widehat{D}_G\varphi =_{def}\widehat{D}_G\varphi$.
\begin{theorem}
For any $G\subseteq N$ and for any $\varphi, \psi\in\mathcal{L}$,
\begin{enumerate}
\item $\models \left\langle\langle G \right \rangle\rangle\Box\varphi\ba D_G\varphi\wedge\left\langle\langle G \right \rangle\rangle\!\bigcirc\!\left\langle\langle G \right \rangle\rangle\Box\varphi$ \label{Box}
  \item $\models \left\langle\langle G \right \rangle\rangle\Diamond\varphi\rightarrow\widehat{D}_G \varphi\vee\left\langle\langle G \right \rangle\rangle\!\bigcirc\!\left\langle\langle G \right \rangle\rangle\Diamond\varphi$ \label{Dnecessary}
  \item $\models D_G\varphi\vee\left\langle\langle G \right \rangle\rangle\!\bigcirc\!\left\langle\langle G \right \rangle\rangle\Diamond\varphi\rightarrow\left\langle\langle G \right \rangle\rangle\Diamond\varphi $\label{Dsufficient}
  \item $\models \left\langle\langle G \right \rangle\rangle\varphi\mathcal{U}\psi\rightarrow\widehat{D}_G \psi\vee(D_G\varphi\wedge\left\langle\langle G \right \rangle\rangle\!\bigcirc\!\left\langle\langle G \right \rangle\rangle\varphi\mathcal{U}\psi)$\label{Unecessary}
  \item $\models D_G\psi\vee(D_G\varphi\wedge\left\langle\langle G \right \rangle\rangle\!\bigcirc\!\left\langle\langle G \right \rangle\rangle\varphi\mathcal{U}\psi)\rightarrow \left\langle\langle G \right \rangle\rangle\varphi\mathcal{U}\psi$ \label{Usufficient}
\end{enumerate}
\end{theorem}
%\comment{
\begin{proof}
For every iCGS $M$, every path $\lambda$ of $M$ and every stage $j\in \mathbb{N}$ on $\lambda$,

\ref{Box}. assume $M, \lambda, j\models D_G\varphi\wedge\left\langle\langle G \right \rangle\rangle\!\bigcirc\!\left\langle\langle G \right \rangle\rangle\Box\varphi$, then  $M, \lambda, j\models D_G\varphi$ and $M, \lambda, j\models \left\langle\langle G \right \rangle\rangle\!\bigcirc\!\left\langle\langle G \right \rangle\rangle\Box\varphi$. By the latter, we get that there is $F^1_G = \langle f^1_i\rangle_{i\in G}$ such that for all $\lambda_1\in \approx_G^j(\lambda)$, for all $\lambda_2\in\mathcal{P}(F_G^1, \lambda_1[0,j])$, $M, \lambda_2,j+1\models\left\langle\langle G \right \rangle\rangle\Box\varphi$. It follows that there is $F^{2\cdot x}_G = \langle f^{2\cdot x}_i\rangle_{i\in G}$ where $x = \lambda_2[0,j+1]$ such that for all $\lambda_3\in \approx_G^{j+1}(\lambda_2)$, for all $\lambda_4\in\mathcal{P}(F_G^{2\cdot x}, \lambda_3[0,j+1])$, for all $t\geq j+1$, $M, \lambda_4, t\models \varphi$. We next construct a new joint strategy $F_G = \langle f_i\rangle_{i\in G}$ based on $F^1_G$ and $F^{2\cdot x}_G$. In order to define $F_G$, we first need the following notation. Let
\vspace{-1mm}
\begin{center}
$X = \{\lambda'[0,j]\stackrel{\alpha}{\rightarrow}w~|~ \lambda'\in \approx_G^j(\lambda), \forall i\in G, \alpha(i) = f^1_i(\lambda'[0,j]) \hbox{ and } w = \delta(\lambda'[j], \alpha)\}$
\end{center}
\vspace{-1mm}
Intuitively, $X$ is the set of all possible outcomes generated by the agents in $G$ taking the next actions specified by $F^1_G$ from a history that is indistinguishable from history $\lambda[0,j]$. We now define the strategy $F_G = \langle f_i\rangle_{ i\in G}$ as follows: For all $h\in H(M)$ and for all $i\in G$,
\begin{displaymath}
  f_i(h) = \left\{\begin{array}{ll}
  f^{2\cdot l}_i(h) & \mbox{ if $\exists l\in X$ such that $l$ is a segment of $h$}\\
   f^1_i(h) & \mbox{ otherwise}
  \end{array}\right.
\end{displaymath}
Note that this strategy is well defined, because if a history $h$ has a segment in $X$, there is only one such segment due to the fact that all histories in $X$ has the same length according to the definition for equivalence relation.

We next show that $F_G$ is just the joint strategy we need to verify $\left\langle\langle G \right \rangle\rangle\Box\varphi$.
That is, for all $\lambda'\in\approx_G^j(\lambda)$, for all $\lambda''\in\mathcal{P}(F_G, \lambda'[0,j])$, we want to prove that for all $s\geq j$, $M, \lambda'', s\models\varphi$. As for any $l\in X$, $|l|>|\lambda'[0,j]|$, then there is no $l\in X$ such that $l$ is a segment of $\lambda'[0,j]$, then by the definition of $F_G$, $F_G(\lambda'[0,j]) = F^1_G(\lambda'[0,j])$, so $\lambda''[0,j+1]\in X$ and  $M, \lambda'', j+1\models \left\langle\langle G \right \rangle\rangle\Box\varphi$. From the later, we get that for all $\lambda^\bullet\in \approx_G^{j+1}(\lambda'')$, for all $\lambda^*\in\mathcal{P}(F_G^{2\cdot y}, \lambda^\bullet[0,j+1])$ where $y =\lambda''[0,j+1]$, for all $t\geq j+1$, $M, \lambda^*, t\models \varphi$. Since $\lambda''[0,j+1]\in X$, then $\lambda^\bullet[0,j+1]\in X$. And by the definition of $F_G$ and the assumption $\lambda''\in\mathcal{P}(F_G, \lambda'[0,j])$, we get $\lambda''\in \mathcal{P}(F_G^{2\cdot y}, \lambda^\bullet[0,j+1])$, so for all $s\geq j+1$, $M, \lambda'', s\models \varphi$. And by the assumption $M, \lambda, j\models D_G\varphi$, we get $M, \lambda'', j\models \varphi$. So for all $s\geq j$, $M, \lambda'', s\models\varphi$, so $M, \lambda, j\models \left\langle\langle G \right \rangle\rangle\Box\varphi$.

The other direction is proved by a similar method in Proposition \ref{rfp}.
%%%%%%%%%%%%%%%%%%%%%%%%%%%%%%%%%%%%%%%%%%%%%%%%%%%%%%%%%%%%%%%%%%%%%%%%%%%%%%%%%%%%%%%%%%%%%%%%%%%%%%%%%%%%%%%%%%%%%%%%%%%%%%%%%%%%%%%%%%%%%%%%%%%%%%%%%%%%%%%%%%%%%%%%%%%%%%%%%%%%%%%%
%%%%%%%%%%%%%%%%%%%%%%%%%%%%%%%%%%%%%%%%%%%%%%%%%%%%%%%%%%%%%%%%%%%%%%%%%%%%%%%%%%%%%%%%%%%%%%%%%%%%%%%%%%%%%%%%%%%%%%%%%%%%%%%%%%%%%%%%%%%%%%%%%%%%%%%%%%%%%%%%%%%%%%%%%%%%%%%%%%%%%%%%

\medskip

\ref{Dnecessary}. assume $M, \lambda, j\models \left\langle\langle G \right \rangle\rangle\Diamond\varphi$, then there is $F_G = \langle f_i\rangle_{i\in G}$ such that for all $\lambda'\in \approx_G^j(\lambda)$, for all $\lambda''\in \mathcal{P}(F_G, \lambda'[0,j])$, there is $k\geq j$ such that $M, \lambda'', k\models\varphi$. Further assume $M, \lambda, j\models D\neg\varphi$, then for all $\lambda^*\in \approx_G^j(\lambda)$, $M, \lambda^*, j\models\neg\varphi$. Then for all $\lambda'\in \approx_G^j(\lambda)$, for all $\lambda''\in \mathcal{P}(F_G, \lambda'[0,j])$, there is $k > j$ such that $M, \lambda'', k\models\varphi$.
We want prove that $M, \lambda,j\models\left\langle\langle G \right \rangle\rangle\!\bigcirc\!\left\langle\langle G \right \rangle\rangle\Diamond\varphi$. It is not hard to show that $F_G$ is just the joint strategy for the both coalition operators.

\medskip

%%%%%%%%%%%%%%%%%%%%%%%%%%%%%%%%%%%%%%%%%%%%%%%%%%%%%%%%%%%%%%%%%%%%%%%%%%%%%%%%%%%%%%%%%%%%%%%%%%%%%%%%%%%%%%%%%%%%%%%%%%%%%%%%%%%%%%%%%%%%%%%%%%%%%%%%%%%%%%%%%%%%%%%%%%%%%%%%%%%%%%%%
%%%%%%%%%%%%%%%%%%%%%%%%%%%%%%%%%%%%%%%%%%%%%%%%%%%%%%%%%%%%%%%%%%%%%%%%%%%%%%%%%%%%%%%%%%%%%%%%%%%%%%%%%%%%%%%%%%%%%%%%%%%%%%%%%%%%%%%%%%%%%%%%%%%%%%%%%%%%%%%%%%%%%%%%%%%%%%%%%%%%%%%%
\ref{Usufficient}. assume $M, \lambda, j\models D_G\psi\vee(D_G\varphi\wedge\left\langle\langle G \right \rangle\rangle\!\bigcirc\!\left\langle\langle G \right \rangle\rangle\varphi\mathcal{U}\psi)$, want $M, \lambda, j\models \left\langle\langle G \right \rangle\rangle\varphi\mathcal{U}\psi$. We next prove this by two cases: either $M, \lambda, j\models D_G\psi$ or $M, \lambda, j\models D_G\varphi\wedge\left\langle\langle G \right \rangle\rangle\!\bigcirc\!\left\langle\langle G \right \rangle\rangle\varphi\mathcal{U}\psi$.

If $M, \lambda, j\models D_G\psi$, then for all $\lambda'\in \approx_G^j(\lambda)$, $M, \lambda', j\models\psi$, then we have that for any $F_G$, for all $\lambda'\in \approx_G^j(\lambda)$, for all $\lambda''\in\mathcal{P}(F_G, \lambda'[0,j])$, $M, \lambda'',j\models\psi$ by $\lambda''[0,j] = \lambda'[0,j]$. Thus,  $M, \lambda, j\models \left\langle\langle G \right \rangle\rangle\varphi\mathcal{U}\psi$.

If $M, \lambda, j\models D_G\varphi\wedge\left\langle\langle G \right \rangle\rangle\!\bigcirc\!\left\langle\langle G \right \rangle\rangle\varphi\mathcal{U}\psi$, then  $M, \lambda, j\models D_G\varphi$ and $M, \lambda, j\models \left\langle\langle G \right \rangle\rangle\!\bigcirc\!\left\langle\langle G \right \rangle\rangle\varphi\mathcal{U}\psi$. By the latter, we get that there is $F^1_G = \langle f^1_i\rangle_{i\in G}$ such that for all $\lambda_1\in \approx_G^j(\lambda)$, for all $\lambda_2\in\mathcal{P}(F_G^1, \lambda_1[0,j])$, $M, \lambda_2,j+1\models\left\langle\langle G \right \rangle\rangle\varphi\mathcal{U}\psi$. It follows that there is $F^{2\cdot x}_G = \langle f^{2\cdot x}_i\rangle_{i\in G}$ where $x = \lambda_2[0,j+1]$ such that for all $\lambda_3\in \approx_G^{j+1}(\lambda_2)$, for all $\lambda_4\in\mathcal{P}(F_G^{2\cdot x}, \lambda_3[0,j+1])$, there is $k \geq j+1$ such that $M, \lambda_4, k\models \psi$ and for all $j+1\leq t <k$, $M, \lambda_4, t\models \varphi$. We next construct a new joint strategy $F_G = \langle f_i\rangle_{i\in G}$ based on $F^1_G$ and $F^{2\cdot x}_G$. In order to define $F_G$, we first need the following notation. Let
\vspace{-1mm}
\begin{center}
$X = \{\lambda'[0,j]\stackrel{\alpha}{\rightarrow}w~|~ \lambda'\in \approx_G^j(\lambda), \forall i\in G, \alpha(i) = f^1_i(\lambda'[0,j]) \hbox{ and } w = \delta(\lambda'[j], \alpha)\}$
\end{center}
\vspace{-1mm}
Intuitively, $X$ is the set of all possible outcomes generated by the agents in $G$ taking the next actions specified by $F^1_G$ from a history that is indistinguishable from history $\lambda[0,j]$. We can now define the strategy $F_G = \langle f_i\rangle_{ i\in G}$ as follows: For all $h\in H(M)$ and for all $i\in G$,
\begin{displaymath}
  f_i(h) = \left\{\begin{array}{ll}
  f^{2\cdot l}_i(h) & \mbox{ if $\exists l\in X$ such that $l$ is a segment of $h$}\\
   f^1_i(h) & \mbox{ otherwise}
  \end{array}\right.
\end{displaymath}
Note that this strategy is well defined, because if a history $h$ has a segment in $X$, there is only one such segment due to the fact that all histories in $X$ has the same length according to the definition for equivalence relation.

We next show that $F_G$ is just the joint strategy we need to verify $\left\langle\langle G \right \rangle\rangle\varphi\mathcal{U}\psi$.
That is, for all $\lambda'\in\approx_G^j(\lambda)$, for all $\lambda''\in\mathcal{P}(F_G, \lambda'[0,j])$, we want to prove that there is $r\geq j$, $M, \lambda'', r\models \psi$ and for all $j\leq s < r$, $M, \lambda'', s\models\varphi$. As for any $l\in X$, $|l|>|\lambda'[0,j]|$, then there is no $l\in X$ such that $l$ is a segment of $\lambda'[0,j]$, then by the definition of $F_G$, $F_G(\lambda'[0,j]) = F^1_G(\lambda'[0,j])$, so $\lambda''[0,j+1]\in X$ and  $M, \lambda'', j+1\models \left\langle\langle G \right \rangle\rangle\varphi\mathcal{U}\psi$. From the later, we get that for all $\lambda^\bullet\in \approx_G^{j+1}(\lambda'')$, for all $\lambda^*\in\mathcal{P}(F_G^{2\cdot y}, \lambda^\bullet[0,j+1])$ where $y =\lambda''[0,j+1]$, there is $k \geq j+1$ such that $M, \lambda^*, k\models \psi$ and for all $j+1\leq t <k$, $M, \lambda^*, t\models \varphi$. Since $\lambda''[0,j+1]\in X$, then $\lambda^\bullet[0,j+1]\in X$. And by the definition of $F_G$ and the assumption $\lambda''\in\mathcal{P}(F_G, \lambda'[0,j])$, we get $\lambda''\in \mathcal{P}(F_G^{2\cdot y}, \lambda^\bullet[0,j+1])$, so there is $r \geq j+1$ such that $M, \lambda'', r\models \psi$ and for all $j+1\leq s <r$, $M, \lambda'', s\models \varphi$. And by the assumption $M, \lambda, j\models D_G\varphi$, we get $M, \lambda'', j\models \varphi$. So there is $r\geq j$, $M, \lambda'', r\models \psi$ and for all $j\leq s < r$, $M, \lambda'', s\models\varphi$, so $M, \lambda, j\models \left\langle\langle G \right \rangle\rangle\varphi\mathcal{U}\psi$.

Thus, in both cases $M, \lambda, j\models \left\langle\langle G \right \rangle\rangle\varphi\mathcal{U}\psi$.

\medskip

The clause \ref{Dsufficient} is proved by a similar method of clause \ref{Usufficient}, while the clause \ref{Unecessary} is proved by a similar method of clause \ref{Dnecessary}.
\end{proof}
%}
The first statement says that a coalition by sharing their knowledge can cooperate to maintain $\varphi$ iff the coalition distributedly knows $\varphi$ at the current stage and there is a joint strategy for this coalition to possess this ability at the next stage. The second statement states that a coalition by sharing their knowledge can eventually achieve $\varphi$ only if either the coalition considers it is possible that $\varphi$ at the current stage or it has a joint strategy to possess this ability at the next stage, while the third statement provides a sufficient condition that  a coalition by sharing their knowledge can eventually achieve $\varphi$ if either it is distributed knowledge among the coalition that $\varphi$ or the coalition can cooperate to achieve this ability at the next stage.
The intuitions behind the last two statements are similar to above two.
%To some extent, the results may be regarded to recover fixed-point characterizations of coalition operators through the interplay of the epistemic and coalition modalities in the context of imperfect information.
In particular, we have the following result for a single agent.
\begin{corollary}
 For any $i\in N$ and any $\varphi,\psi\in\mathcal{L}$,
\begin{enumerate}
\item $\models \left\langle\langle i \right \rangle\rangle\Box\varphi\ba K_i\varphi\wedge\left\langle\langle i \right \rangle\rangle\!\bigcirc\!\left\langle\langle i \right \rangle\rangle\Box\varphi$ \label{Box}
  \item $\models \left\langle\langle i \right \rangle\rangle\Diamond\varphi\rightarrow\widehat{K}_i \varphi\vee\left\langle\langle i \right \rangle\rangle\!\bigcirc\!\left\langle\langle i \right \rangle\rangle\Diamond\varphi$ \label{Dnecessary}
  \item $\models K_i\varphi\vee\left\langle\langle i\right \rangle\rangle\!\bigcirc\!\left\langle\langle i \right \rangle\rangle\Diamond\varphi\rightarrow\left\langle\langle i \right \rangle\rangle\Diamond\varphi $\label{Dsufficient}
  \item $\models \left\langle\langle i \right \rangle\rangle\varphi\mathcal{U}\psi\rightarrow\widehat{K}_i\psi\vee(K_i\varphi\wedge\left\langle\langle i \right \rangle\rangle\!\bigcirc\!\left\langle\langle i \right \rangle\rangle\varphi\mathcal{U}\psi)$\label{Unecessary}
  \item $\models K_i\psi\vee(K_i\varphi\wedge\left\langle\langle i \right \rangle\rangle\!\bigcirc\!\left\langle\langle i \right \rangle\rangle\varphi\mathcal{U}\psi)\rightarrow \left\langle\langle i \right \rangle\rangle\varphi\mathcal{U}\psi$ \label{Usufficient}
\end{enumerate}
\end{corollary}

%Finally, we would like to remark that on the one hand, such results shed light on the interplay of group knowledge and coalition ability under imperfect information and perfect recall; on the other hand, as pointed by~\cite{Belar15}, the interest of such fixed-point characterizations lies in the fact that similar characterizations feature prominently in decision procedures for the satisfiability and model checking problems of temporal logics~\cite{bolander2006tableau}. We envisage they may play a similar role for ATL with imperfect information and perfect recall.

\section{Related Work}
In recent years, there are many logical formalisms for reasoning about coalition abilities and strategic interactions in MAS. \cite{Ditmarsch2015,herzig2014logics} provide a latest survey of this topic. In this following, we will review several works which are most related to ATL with imperfect information and perfect recall.

%Alternating-time temporal epistemic logic (ATEL) is known to be the first logic to extend ATL with epistemic operators~\cite{van2003cooperation}. ATEL is a succinct and very powerful language to express the complex properties of MAS with imperfect information, but since the original semantics are still used for both epistemic and coalitional operators, the full interaction between agents' knowledge and their strategies is not captured. Moreover, as agents' strategies do not take their indistinguishability relation into account, ATEL has the issue of {\it uniformity of strategies} which is first pointed out in~\cite{jamroga2003some} and fully discussed in~\cite{jamroga2004agents}. Since then, almost all of the epistemic ATL-style logics take uniform strategies into consideration. In this paper we continue this line of work and focus on uniform strategies based on the consideration that in imperfect information games what agents are able to achieve by choosing their strategies is based on their knowledge.

In the context of imperfect information, several semantic variants have been proposed for ATL based on different interpretations of agents' ability~\cite{aagotnes2007alternating,jamroga2004agents,schobbens2004alternating,van2003cooperation}. In particular, \cite{bulling2014comparing,jamroga2011comparing} provide formal comparisons of validity sets for semantic variants of ATL. Similar to Bulling {\it et al.}'s no forgetting semantics~\cite{bulling2014agents}, our semantics is also history-based w.r.t a path and an index on the path,
%an agent's strategy is also affected by its memory ability. For instance, agents may recall the entire history of the game (perfect recall) or just remember the current state (imperfect recall also called memoryless) \cite{jamroga2004agents,schobbens2004alternating}. In terms of this, uniform strategies are classified into two basic types: perfect recall uniform strategies and memoryless (also imperfect recall) uniform strategies. The notion of perfect recall in standard semantics of ATL with imperfect information are state-based requiring an agent remembers the past states. However, Bulling {\it et al} in \cite{jamroga2003some} argue that the standard perfect recall semantics of ATL has counterintuitive effects: agents may forget the past despite using perfect recall strategy, which motivates a variant of the perfect recall semantics called the truly perfect recall semantics. The main trait of this semantics is that it keeps track of the history. Instead of starting with the current state of the game as in the standard semantics, they take paths that describe the play from the beginning.
but there are fundamental differences. First of all, we consider a finer notion of perfect recall by taking both past states and actions into considerations to deal with situations where different actions may have the same effects. Secondly, our notion of group uniform strategies is defined in terms of distributed knowledge instead of general knowledge as we assume that when a set of agents form a coalition, they are able to share their knowledge before cooperating to ensure a goal.% It turns out that our semantics is sufficient to overcome the counterintuitive phenomenons in ATL with no forgetting semantics.

 %With imperfect information, coalition operators implicitly include group knowledge.
% As pointed by~\cite{herzig2014logics}, one of difficulties for ATL with imperfect information is that which kind of group knowledge is required for a group to achieve some goal.
Several epistemic-ATL style logics have been proposed to investigate the interaction of group knowledge and coalition ability~\cite{bulling2014comparing,vanDitmarsch14,schobbens2004alternating,guelev2011alternating,huang2016strengthening}. %To the best of our knowledge, there are three basic types: uniform strategies based on general knowledge~\cite{schobbens2004alternating,bulling2014comparing}, uniform strategies based on common knowledge~\cite{vanDitmarsch14} and uniform strategies based on distributed knowledge~\cite{vanDitmarsch14}.
 %We consider a coalition as a set of agents who are able to share their own knowledge before acting.
In particular, the most relevant works are~\cite{vanDitmarsch14,guelev2011alternating}. Specifically,
\cite{guelev2011alternating} presents a variant of ATL with knowledge, perfect recall and past. Different from our motivation,  they use the distributed knowledge of coalitions so as to have a decidable model-checking problem.
\cite{vanDitmarsch14} proposes three types of coalition operators to specify different cases of how all agents in the coalition cooperate to enforce a goal. Among them, the communication strategy operator $\left\langle\langle G\right \rangle\rangle_c$ captures the intuition behind our coalition operator.  %van Ditmarsch {\it et al} proposed two versions of epistemic ATL, namely uATL, euATL.In particular, the semantics for euATL is history-based.
Specifically, we have the following correspondence.
\vspace{-1mm}
\begin{proposition}
Given an iCGS $M$, a path $\lambda$ of $M$ and a stage $j\in \mathbb{N}$ on $\lambda$, let $\varphi$ be any formula of the form $\!\bigcirc\!\psi$, $\Box\psi$ or $\psi_1\mathcal{U}\psi_2$,
$M, \lambda, j\models \left\langle\langle G\right \rangle\rangle\varphi \mbox{ iff } M, \lambda[0,j]\models_{euATL} \left\langle\langle G\right \rangle\rangle_c\varphi$
\end{proposition}
\vspace{-1mm}
However, their work is different from ours in the following aspects: firstly, they propose two epistemic versions of ATL, namely uATL and euATL, to address the issue of uniformity of strategies in the combination of strategic and epistemic systems, while we introduce a new semantics without adding new operators to the language to explore the interplay of epistemic and coalitional operators; secondly, their results mainly focus on the relations and logical properties of three coalition ability operators, while we investigate fixed-pointed characterizations for the interplay of distributed knowledge and coalition operators which is not involved in~\cite{vanDitmarsch14};
thirdly, their meaning of coalition is more subtle than ours. Except the communication strategy operator,
the comparison with the other two strategy operators is less straightforward since they are based on assumptions of coalitions without sharing knowledge. We hope to understand them better in the future.

Finally, it is also worth mentioning that \cite{herzig2006knowing} adopts a similar meaning of coalition so as to capture the notion of ``knowing how to play''. Besides the different motivations, that work is based on STIT framework and just considers one-step uniform strategies without investigating the interplay of epistemic and coalitional operators.

\section{Conclusion}
In this paper, we have proposed new semantics for ATL with imperfect information and perfect recall to explore the interplay of the knowledge shared by a group of agents and its coalition abilities. Compared to existing alternative semantics, we have showed that our semantics can not only preserve the desirable properties of coalition ability in traditional coalitional logics, but also provide a finer notion of perfect recall requiring an agent remembers the past states as well as the past actions. More importantly, we have investigated the interplay of epistemic and coalitional operators.

In the future we intend to study the computational complexity of ATL with this new semantics, such as the model-checking problem. %For instance, \cite{schobbens2004alternating} has showed that most models with perfect recall and imperfect information have an undecidable model-checking problem. Moreover, \cite{diaconu2012model} has proven that the model-checking problem for epistemic ATL with strategies based on common knowledge is undecidable. Based on previous results~\cite{diaconu2012model,schobbens2004alternating}, we expect similar undecidability result hold under our setting. %,  although our models for ATL with imperfect information and perfect recall are different from theirs. If it is the case, it would also be interesting to explore the decidable segment of this variant.
In this paper, we have investigated how knowledge sharing within a group of agents contributes to its coalitional ability. %logical properties for the epistemic ATL with uniform strategies based on distributed knowledge. %and obtained fixed-point characterizations of coalition operators through the interplay of distributed knowledge and coalition operators.
 This work can be seen as an attempt towards the question: which kind of group knowledge is required for a group to achieve some goal in the context of imperfect information. We believe that it is an interesting question for further investigation by considering other cases such as group without knowledge sharing or with partial knowledge sharing~\cite{huang2016strengthening}.

\section*{Acknowledgments}
We are grateful to Heng Zhang for his valuable help, and special thanks are due to three anonymous referees for their insightful comments. This research was partially supported by A key project of National Science of China titled with A study on dynamic logics for games (15AZX020).

\bibliographystyle{splncs03}
%\bibliography{Guifei}

\end{document}